\documentclass[reqno,11.5pt]{article}
\usepackage[normalem]{ulem}

\usepackage{amssymb}
\usepackage{tikz}
\usepackage{amsmath}
\usepackage{latexsym}
\usepackage{mathrsfs}
\usepackage{amsthm}
\usepackage{verbatim}
\usepackage{graphicx}
\usepackage{epstopdf}
\usepackage{epsfig}
\usepackage{color}
\usepackage{float}
\usepackage{titlesec}
\usepackage{bm}
\usepackage{amsfonts}
\usepackage{fancyhdr}
\usepackage{amscd}
\usepackage{cases}
\usepackage{amsmath,stackrel}

\usepackage{mathtools}

\usepackage{multirow}

\makeatletter
\@addtoreset{equation}{section}
\makeatother

\usepackage{geometry}
\usepackage{tikz}
\usetikzlibrary{er}

\usepackage[round,authoryear]{natbib}
\usepackage[unicode=true,pdfusetitle,
bookmarks=ture,bookmarksnumbered=false,bookmarksopen=false,
breaklinks=false,pdfborder={0 0 1},backref=false,colorlinks=false]{hyperref}

\title{MgNet: A Unified Framework \\ of Multigrid and Convolutional Neural Network}

\author{Juncai He\footnotemark[1] ~and Jinchao Xu\footnotemark[2]}
\date{} 

\newcommand{\classmap}{H}
\newcommand{\linearize}{H_0}

\newcommand{\FC}{{\mathcal D}}

\usepackage{algorithm,algpseudocode,float}
\usepackage{lipsum}

\makeatletter
\newenvironment{breakablealgorithm}
{
	\begin{center}
		\refstepcounter{algorithm}
		\hrule height.8pt depth0pt \kern2pt
		\renewcommand{\caption}[2][\relax]{
			{\raggedright\textbf{\ALG@name~\thealgorithm} ##2\par}%
			\ifx\relax##1\relax 
			\addcontentsline{loa}{algorithm}{\protect\numberline{\thealgorithm}##2}%
			\else 
			\addcontentsline{loa}{algorithm}{\protect\numberline{\thealgorithm}##1}%
			\fi
			\kern2pt\hrule\kern2pt
		}
	}{
		\kern2pt\hrule\relax
	\end{center}
}
\makeatother

\usepackage{tikz}   
\usetikzlibrary{positioning}
\usepackage{amsmath}
\usepackage{amsfonts} 
\usepackage{amssymb}
\usepackage{bm}
\newtheorem{theorem}{Theorem}
\newtheorem{remark}{Remark}

\newtheorem{lemma}{Lemma}

\newcommand{\relu}{\mbox{{\rm ReLU}}}

\begin{document}
	
\maketitle 
\renewcommand{\thefootnote}{\fnsymbol{footnote}} 
\footnotetext[1]{School of Mathematical Sciences, Peking University, Beijing 100871, China (juncaihe@pku.edu.cn).} 
\footnotetext[2]{Department of Mathematics, The Pennsylvania State University, University Park, PA 16802, USA (xu@math.psu.edu).} 

\begin{abstract}
	We develop a unified model, known as MgNet, that simultaneously recovers some 
	convolutional neural networks (CNN) for image classification and multigrid (MG) 
	methods for solving discretized partial differential equations (PDEs).  This model is 
	based on close connections that we have observed and uncovered between the CNN 
	and MG methodologies.  For example, pooling operation and feature extraction in CNN 
	correspond directly to restriction operation and iterative smoothers in MG, respectively.
	As the solution space is often the dual of the data space in PDEs, the analogous concept
	of feature space and data space (which are dual to each other) is introduced in CNN.   
	With such connections and new concept in the unified model, the function of various
	convolution operations and pooling used in CNN can be better understood. 
	As a result, modified CNN models (with fewer weights and hyperparameters) 
	are developed that exhibit competitive and sometimes better performance in 
	comparison with existing CNN models when applied to both CIFAR-10 and CIFAR-100 data sets.
\end{abstract}


\section{Introduction}
This paper is devoted to the study of convolutional neural networks
(CNN) \cite{lecun1998gradient, krizhevsky2012imagenet, goodfellow2017deep} in machine learning by exploring their
relationship with multigrid methods for numerically solving partial
differential equation \cite{xu1992iterative,xu2002method, hackbusch2013multi}.  CNN has
been successfully applied in many areas, especially computer vision
\cite{lecun2015deep}.
Important examples of CNN include the LeNet-5 model of LeCun et al. in
1998 \cite{lecun1998gradient}, the AlexNet of Hinton et el in 2012
\cite{krizhevsky2012imagenet}, Residual Network (ResNet) of K. He et
al. in 2015 \cite{he2016deep} and other variants of CNN in
\cite{simonyan2014very, szegedy2015going,
	huang2017densely}.
Given the great success of CNN models, it is of both theoretical and
practical interest to understand why and how CNN works.

In 1990s, the mathematical analysis of DNN mainly focus on
the approximation properties for  DNN and CNN models. 
The first approximation results for DNN are obtained for a feedforward neural network with a single
hidden layer separately in \cite{hornik1989multilayer} and \cite{cybenko1989approximation}.
From 1989 to 1999, many results about the so-called expressive 
power of single hidden neural networks are derived 
\cite{barron1993universal, ellacott1994aspects, pinkus1999approximation}.
In recent, many new DNN structure with ReLU \cite{nair2010rectified} activation function
have been studied in connection with: wavelets \cite{shaham2018provable},
finite element \cite{he2018relu}, sparse grid \cite{montanelli2017deep} 
and polynomial expansion \cite{e2018exponential}.
By using a connection of CNN and DNN that a convolution with large enough
kernel can recover any linear mapping, \cite{zhou2018universality} presents
an approximate result with convergence rate by deep CNNs 
for functions in the Sobolev space $H^r(\Omega)$ with $r > 2 + d/2$, see also 
most recent result of \cite{siegel2019approximation}.

These function approximation theories for deep learning, 
are far from being adequate to explain why deep neural network, 
especially for CNN, works and to understand the efficiency of 
some successful models such as ResNet. 
One goal of this paper is to offer some mathematical insights into CNN by using ideas
from multigrid methods and by developing a theoretical framework for
these two methodologies from different fields. Furthermore, such insight is used to 
develop more efficient CNN models.

In the existing deep learning literature, ideas and techniques from multigrid 
methods have been used for the development of efficient
deep neural networks.  As a prominent example,  the 
ResNet  and iResNet developed in \cite{he2016deep, he2016identity},
are motivated in part by the hierarchical use of
``residuals'' in multigrid methods as mentioned by the authors.
As another example, in \cite{ronneberger2015u, milletari2016v}, 
a CNN model with almost the same structure as the V-cycle multigrid is 
proposed to deal with volumetric medical image segmentation and 
biomedical image segmentation.  More recently, multi-resolution images 
have been used as the input into the neural network in \cite{haber2017learning}. 
\cite{ke2016multigrid} use different net works to deal with 
multi-resolution images separately with a CNN to glue them together. 

A dynamic system viewpoint has also been explored in many papers such as 
\cite{haber2017learning,e2017a, lu2018beyond}  
to understand the iterative structure in ResNet type models such as 
the iResNet model in \cite{he2016identity}:
\begin{equation}\label{resnet-form}
x^{i} = x^{i-1} + f^i(x^{i-1}).
\end{equation}
Such an idea is further explored by \cite{li2017a} to use some flow
model to interpret the date flow in ResNet as the solution of
transport equation following the characteristic line.
\cite{chang2017multi} proposes a multi-level training algorithm for
the ResNet model by training a shallow model first and then prolongating its
parameters to train a deeper model.  \cite{lu2018beyond} uses the
idea of time discretization in dynamic systems to interpret PloyNet
\cite{zhang2017polynet}, FractalNet \cite{larsson2016fractalnet} and
RevNet \cite{gomez2017reversible} as different time discretization
schemes.  Then they propose the LM-ResNet based on the idea of linear
multi-step schemes in numerical ODEs with a stochastic learning
strategy.  \cite{long2018pde1, long2018pde2} construct the PDE-Net
models to learn PDE model from data connecting discrete
differential operators and convolutions.

In a different direction, new multigrid methods for numerical PDEs can
be motivated by deep learning.  For example, in
\cite{katrutsa2017deep} a Deep Multigrid Method is proposed where the
restriction and prolongation matrices with a given sparsity pattern
are trained by minimizing the Frobenius norm of a large power of the
multigrid error propagation matrix with a sampling technique similar
to what is used in machine learning.  In \cite{hsieh2018learning}, a
linear U-net structure is proposed as a solver for linear PDEs on the
regular mesh.

In this paper, we explore the connection between multigrid and convolutional neural networks,
in several directions.  First of all, we view the 
multi-scale of images used in CNN  as
piecewise (bi-)linear functions as used in multigrid methods, and we 
relate the pooling
operation in CNN with the restriction operation in multigrid.

To examine further connections between CNN and
multigrid, we introduce the so-called data and feature space for CNN,
which is analogous to the function space and its duality in the theory
of multigrid methods \cite{xu2017algebraic}.  With this new concept
for CNN, we propose the data-feature mapping model in every grid as
\begin{equation}\label{eq:fmapping}
A(u) = f,
\end{equation}
where $f$ belongs to the data space and $u$ belongs to the feature space. 
The feature extraction process can then be obtained through an
iterative procedure
for solving the above system, namely
\begin{equation}\label{BAmapping}
u^{i} = u^{i-1} + B^{i} (f^{} - A^{}(u^{i-1})), \quad i = 1:\nu,
\end{equation}
with $u \approx u^{\nu}$. The above iterative scheme \eqref{BAmapping}
can be interpreted as both the feature extraction step in ResNet type models
and the smoothing step in multigrid method.

Using the above observations and new concepts, we develop a unified framework, called MgNet, that
simultaneously recovers some convolutional neural networks and multigrid methods.
Furthermore, we establish connections between several ResNet type
models using the MgNet framework.
We provide improvements/generalizations of several ResNet type
models that are as competitive as and sometimes  more efficient than existing models, as 
demonstrated by  numerically
experiments for both CIFAR-10 and CIFAR-100 \cite{krizhevsky2009learning}.

The remaining sections are organized as follows. In \S~\ref{sec:MLbasics}, we 
introduce some notation and preliminary results in supervised learning especially 
for image classification problem. In \S~\ref{sec:spaces}, we present the idea that
we need to distinguish the data and feature space in CNN models and introduce some
related mappings. In \S~\ref{sec:functions}, we explore the structures and operators when
we consider images as (bi-)linear functions in multilevel grids. In \S~\ref{sec:mg}, we 
introduce multigrid by splitting it into two phases. In \S~\ref{sec:mgnet}, we give an abstract form 
of MgNet as a framework for multigrid and convolutional neural network with details.
In \S~\ref{sec:CNNs}, we introduce some classical CNN structures with rigorous mathematical 
definition. In \S~\ref{sec:relation}, we construct some relations and connections between MgNet and classic models. In \S~\ref{sec:numerics}, we present some numerical results to show the efficiency of MgNet. In \S~\ref{sec:conclusion} we give concluding remarks.

\section{Supervised learning on image classification}\label{sec:MLbasics}
We consider a basic machine learning problem for classifying a
collection of images into $\kappa$ distinctive classes.  As an
example, we consider a two-dimensional image which is usually
represented by a tensor
$$
f\in \mathbb  \FC := \mathbb{R}^{m\times n\times c}.
$$
Here 
\begin{equation}
\label{data-c}
c=
\left\{
\begin{array}[rl]{rl}
1 & \mbox{for grayscale image},\\    
3 & \mbox{for color image}.
\end{array}
\right.
\end{equation}

A typical supervised machine learning problem begins with a data set (training data)
$$
D := \{(f_i, y_i)\}_{i=1}^N,
$$ 
with
$$
\{f_i\}_{i=1}^N \subset \FC,
$$
and $y_i \in \mathbb{R}^{\kappa}$ is the label for data $f_i$, with
$[y_i]_j$ as the probability for $f_i$ in classes $j$.

Roughly speaking, a supervised learning problem can be thought as data fitting
problem in a high dimensional space $\FC$.
Namely, we need to find a mapping
$$
\classmap:  \mathbb R^{m\times n\times c}\mapsto \mathbb R^\kappa,
$$
such that, for a given $f\in  \FC$, 
\begin{equation}\label{eq:idealouput}
\classmap(f)\approx e_i\in \mathbb R^\kappa,
\end{equation}
if $f$ is in class $i$, for $1\le i\le \kappa$. 
For the general setting above, we use a probatilistic model for understanding the
output $\classmap(f) \in \mathbb{R}^{\kappa}$ as a discrete 
distribution on $\{1, \cdots,\kappa\}$, with $[\classmap (f)]_i$ as the probability
for $f$ in the class $i$, namely
\begin{equation}
\label{distrib}
0 \le [\classmap(f)]_i \le 1,\quad 
\sum_{i=1}^\kappa  [\classmap(f)]_i=1. 
\end{equation}
At last, we finish our model with a simple strategy to choose
\begin{equation}\label{eq:maxchoose}
\mathop{\arg\max}_{i}\{[\classmap(f)]_i~:~ i = 1:\kappa\},
\end{equation}
as the label for a test data $f$, which ideally is close to
\eqref{eq:idealouput}.  The remaining key issue is the construction of
the classification mapping $\classmap$.

The main step in the construction of $\classmap$ is to 
construct a nonlinear mapping
\begin{equation}
\label{linearize}
\linearize: \FC \mapsto V_J,
\end{equation}
with 
\begin{equation}
\label{VJ}
V_J = \mathbb R^{m_J\times n_J\times c_J}. 
\end{equation}
To be consistent with the notation for CNN which will be described below,
here the subscript $J$ refers to the number of 
coarsening girds in CNN. 
Roughly speaking, the map $\linearize$ plays two roles.  The first role
is to conduct a dimensionality reduction, namely
$$
m_Jn_Jc_J\ll  mnc.
$$
The second role is to map a complicated set of data into a set of data
that are linearly separable. As a result, the simple logistic regression 
procedure can be applied.

The first step in a logistic regression is to introduce a linear mapping:
\begin{equation*}
\Theta: \FC \to\mathbb{R}^{\kappa} ,
\end{equation*}as 
\begin{equation}\label{thetamap}
\Theta(x)=Wx+b,
\end{equation}
where $W=(w_{ij})\in\mathbb{R}^{(m_J \times n_J \times c_J)\times \kappa}$, 
$b\in\mathbb{R}^{\kappa}$.

We then use the soft-max function.
\begin{equation}
\label{softmax}
[S(z)]_i=[{\rm Solftmax}(z)]_i= \frac{e^{z_i}}{\sum_{j} e^{z_j}},
\end{equation}
to obtain a logistic regression model 
\begin{equation}
\label{eq:log_reg}
S \circ \Theta: \mathbb R^{m_J\times n_J\times c_J}\mapsto \mathbb R^\kappa.
\end{equation}

By combining the nonlinear mapping $\classmap$ in \eqref{linearize}
and the logistic regression \eqref{eq:log_reg}, we obtain the following classifier:
\begin{equation}
\label{classifier}
\classmap=  S\circ \Theta\circ \linearize.
\end{equation}

Given the model \eqref{classifier},
we finish the training phase with solving the next optimization 
problem:
\begin{equation}
\label{eq:3}
\min \sum_{j=1}^Nl(\classmap(f_j),y_j),
\end{equation}
where
Here $l(\classmap(f_j),y_j)$ is a  loss function that measures the
predicted result $\classmap(f_j)$ and the real label $y_j$. 
In logistic regression, 
the following cross-entropy loss function is often used
$$
l(\classmap(f), y) = \sum_{i=1}^\kappa -[y]_i \log [\classmap(f)]_i.
$$

\section{Data space, feature space and relevant mappings}\label{sec:spaces}
Given a data
\begin{equation}
\label{data-f}
f \in \mathbb{R}^{m \times n \times c}, 
\quad \text{or}\quad [f]_i \in
\mathbb{R}^{m\times n}, \quad i = 1:c,
\end{equation}
where $m\times n$ is called the spatial dimension and $c$ is the 
channel dimension.

For the given data $f$ in \eqref{data-f}, we look for some
feature vector, denoted by $u$,  associated with $f$:
\begin{equation}
\label{u}
u \in \mathbb{R}^{m \times n \times h}.
\end{equation}
We make an assumption that the data $f$ and feature $u$ are related by a mapping 
(which can be either linear or nonlinear)
\begin{equation}
\label{u}
A:  \mathbb{R}^{m \times n \times h}\mapsto \mathbb{R}^{m \times n \times c}, 
\end{equation}
so that
\begin{equation}
\label{Auf}
A(u)=f. 
\end{equation}
A mapping 
\begin{equation*}
B : \mathbb{R}^{m\times n\times c} \mapsto \mathbb{R}^{m \times n\times h},
\end{equation*}
is called a feature extractor if $B \approx A^{-1}$ and 
\begin{equation}\label{vBf}
v = B(f),
\end{equation}
is such that $v \approx u$.

The data-feature relationship \eqref{Auf} or \eqref{vBf} is not
unique.   Different relationships give rise to different features. 
We can view the data-feature relationship given in \eqref{Auf}
as a model that we propose.  Here the mapping $A$, which can be either
linear or nonlinear, is unknown and needs to be trained.  

We point out  that the data space and feature space may have different
numbers of channels.

\subsection{A special linear mapping: convolution}
One important class of linear mapping is the so-called convolution:
$$
\theta: \mathbb{R}^{m\times n\times c} \mapsto \mathbb{R}^{m\times n \times h},
$$
that can be defined by 
\begin{equation}\label{conv-1}
[\theta(f)]_{t} = \sum_i^{c}K_{i,t} \ast [f]_i + b_t 
\bm{1}  \in \mathbb{R}^{m\times n}, \quad t = 1:h,
\end{equation}
where $\bm{1}  \in \mathbb{R}^{m\times n} $ is a 
$m\times n$ matrix with all elements being $1$,
and for $g \in \mathbb{R}^{m\times n}$
\begin{equation}\label{con1}
[K \ast g]_{i,j} = \sum_{p,q=-k}^k K_{k+1+p,k+1+q} g_{i + p, j + q}, \quad i=1:m, j = 1:n.
\end{equation}
The coefficients in \eqref{con1} constitute  a kernel matrix
\begin{equation}
K \in \mathbb{R}^{(2k+1) \times (2k+1)},
\end{equation}
where $k$ is often taken as small integers. 
Here padding means how to choose $ X_{i+ p, j + q}$ 
when $(i+ p, j + q)$ is out of $1:m$ or $1:n$. 
Those next three choices are often used
\begin{equation}\label{eq:padding}
f_{i + p, j + q} = \begin{cases}
0,  \quad &\text{zero padding}, \\
f_{(i + p)\pmod{m}, (s + q)\pmod{n}},  \quad &\text{periodic padding}, \\
f_{|i-1 +p|, |j -1  +q|},  \quad &\text{reflected padding}, \\
\end{cases}
\end{equation}
if 
\begin{equation}
i + p \notin \{1, 2, \dots, m\} ~\text{or} ~  j+ q \notin \{1, 2, \dots, n\}.
\end{equation}
Here $ d \pmod m \in \{1, \cdots, m\} $  means the remainder when $d$ is divided by $m$.


If we formally write 
\begin{equation}
f=
\begin{pmatrix}
f_1\\
\vdots\\
f_c  
\end{pmatrix}.
\end{equation}
We can then write the operation \eqref{conv-1} as
\begin{equation}
\label{eq:4}
\theta(f)=K\ast f+b,
\end{equation}
where 
$$
K=(K_{ij})\in \mathbb R^{[(2k+1)\times(2k+1)]\times h\times c},
$$
and 
$$
\bm{b}=\bm{1}_{m\times n}\otimes b.
$$

The operation \eqref{eq:4} is also called a convolution with stride 1. More generally, 
given an integer $s\ge1$, a convolution with stride $s$ for $f \in \mathbb{R}^{m\times
	n}$ is defined as:
\begin{equation}\label{stride}
[K \ast_s f]_{i,j} = \sum_{p,q=-k}^k K_{p,q} f_{s(i-1)+1 + p, s(j-1)+1 + q},  
\quad i = 1: \lceil  \frac{m}{s}\rceil , j = 1: \lceil  \frac{n}{s}\rceil.
\end{equation}
Here $ \lceil  \frac{m}{s}\rceil$ denotes the smallest integer that greater than $\frac{m}{s}$.
In CNN, we often take $s=2$.

\subsection{Some linear and nonlinear mappings and extractors}
A data-feature map $A$ and feacture extractor $B$ can be either
linear or nonlinear.   The nonlinearity can be obtained from
appropriate application of an activation function
\begin{equation}
\label{act}
\sigma: \mathbb{R} \to \mathbb{R} .
\end{equation}
In this paper, we mainly consider a special activation function, known 
as the {\it rectified linear unit} (ReLU), which is defined by
\begin{equation}
\label{relu}
\sigma(x)= \relu(x) :=\max(0,x), \quad x\in\mathbb{R}. 
\end{equation}
By applying the function to each component, we can extend this
\begin{equation}
\label{vector-act}
\sigma:\mathbb R^{m\times n\times c}\mapsto \mathbb R^{m\times n\times c}.  
\end{equation}

A linear data-feature mapping can simply given by a convolution as in \eqref{con1}:
\begin{equation}
\label{linearA}
A(u)=\xi\ast u.
\end{equation}
A nonlinear mapping can be given by compositions of convolution and
activation functions:
\begin{equation}
\label{nonlinearA}
A=\xi\circ\sigma\circ\eta ,
\end{equation}

and 
\begin{equation}
\label{extractor}
B=\sigma\circ \gamma \circ\sigma  .
\end{equation}
Here $\xi$, $\eta$ and $\gamma$ are all 
appropriate convolution mappings.

\subsection{Iterative feature extraction schemes}
One key idea in this paper is that we consider different iterative
processes to approximately solve \eqref{Auf} and relate them to
many existing popular CNN models. Here, let us assume that 
the feature-data mapping \eqref{Auf}  is given as a linear form \eqref{linearA}.
We next propose some iterative schemes to solve \eqref{Auf}
for an appropriately chosen $u^0$.
\begin{itemize}
	\item Residual correction method, 
	\begin{equation}\label{eq:smoothB}
	u^{i} = u^{i-1} + B^{i}(f- A(u^{i-1})), \quad i=1:\nu.
	\end{equation}
	Here $B^i$ can be chosen as linear like $B^{i}(f) = \eta^i \ast f$ or nonlinear
	like \eqref{extractor}. The reason why $B^{i}$ is taken
	the nonlinear form as in \eqref{extractor} will be discussed later based on our
	main discovery about the relationship
	between MgNet and iResNet as discussed in \S~\ref{sec:CNNs} and \S~\ref{sec:relation}. 
	We refer to \cite{xu1992iterative}
	for more discussion on iterative schemes in the form of \eqref{eq:smoothB}.
	\item Semi-iterative method for accelerating the residual correction iterative scheme,
	\begin{equation}\label{eq:multi}
	u^{i} = \sum_{j=0}^{i-1}  \alpha_{j}^i\left( u^{j} + B^{i}_j(f - A(u^j)) \right), \quad i=1:\nu,
	\end{equation}
	where $\alpha_j^i \ge 0$ and $\sum_{j=0}^{i-1}  \alpha_{j}^i = 1$.
	Letting the residual $r^j = f - A(u^{j})$ for $j=0:i$, 
	the following iterative scheme for $r^j$ is implied by \eqref{eq:multi}
	\begin{equation}\label{eq:multi-res}
	r^{i} = \sum_{j=0}^{i-1}\alpha^i_j(I - AB^{i}_j)(r^j),
	\end{equation}
	because of the linearity of $A$. This scheme is analogous to the 
	DenseNet \cite{huang2017densely} which will be discussed more in \S~\ref{sec:CNNs} and below.
	More discussion on semi-iterative method for linear system 
	can be found in \cite{hackbusch1994iterative, golub2012matrix}.
	\item Chebyshev semi-iterative method, 
	\begin{equation}\label{eq:chebysev-semi}
	u^{i} = \omega^i\left(u^{i-1} + B^{i}\left(f- A(u^{i-1})\right)\right)+ (1- \omega^i) u^{i-2},\quad i=1:\nu.
	\end{equation}
	The above scheme can be obtained from the above semi-iterative form
	by applying the Chebyshev polynomial theory \cite{hackbusch1994iterative, golub2012matrix}. 
	Similar to the previous case, considering the iterative form of the residual $r^j = f - A(u^{j})$,
	\eqref{eq:chebysev-semi} implies that
	\begin{equation}
	r^{i} = \omega^i r^{i-1} + (1-\omega^i)r^{i-2} - AB^i r^{i-1}.
	\end{equation}
	This scheme corresponds to the LM-ResNet in \cite{lu2018beyond} 
	which was obtained as a linear multi-step scheme for some underlying ODEs.
\end{itemize}

\section{Piecewise (bi-)linear functions on multilevel grids}\label{sec:functions}
An image can be viewed as a function on a grid.  Images
with different resolutions can then be viewed as functions on grids of
different sizes.  The use of such multiple-grids is a main technique
used in the standard multigrid method for solving discretized partial
differential equations \cite{xu1992iterative, xu2002method}, 
and it can also be interpreted as a main ingredient used in
convolutional neural networks (CNN). 

Without loss of generality, for simplicity, we assume that the initial
grid, $\mathcal T$, is of size
$$
m=2^{s}+1, n=2^{t}+1 ,
$$
for some integers $s, t\ge 1$.
Starting from $\mathcal T_1=\mathcal T$,  we consider a sequence of
coarse grids (as depicted in Fig.~\ref{mgrid} with $J=4$):
\begin{equation}
\label{grids}
\mathcal T_1,~ \mathcal T_2,~ \ldots,~ \mathcal T_J,
\end{equation}
such that ${\cal T}_\ell$ consist of $m_\ell\times n_\ell$ grid
points, with 
\begin{equation}
\label{mn-ell}
m_\ell=2^{s-\ell+1}+1, \quad  n_\ell=2^{t-\ell+1}+1.   
\end{equation}
\begin{figure}[!htbp]\label{mgrid}
	\begin{center}
		\includegraphics[width=0.15\textwidth]{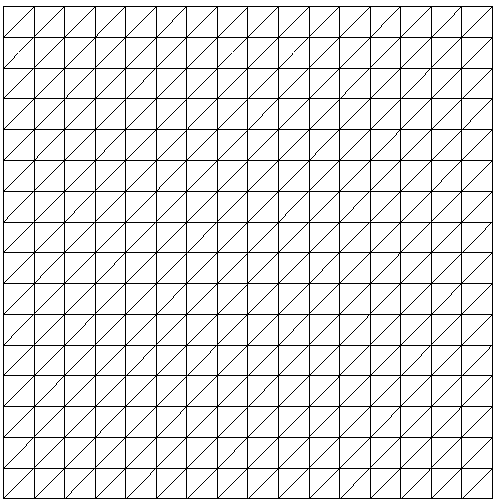} \quad 
		\includegraphics[width=0.15\textwidth]{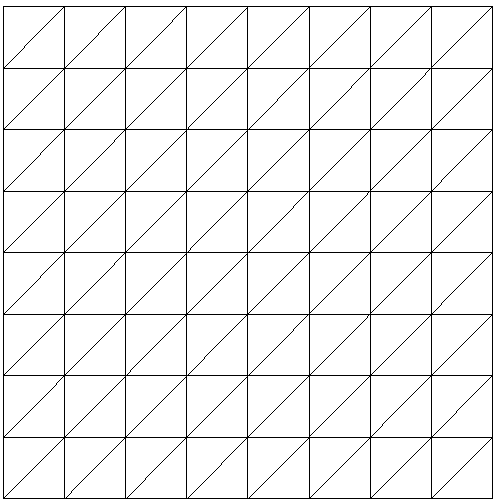} \quad 
		\includegraphics[width=0.15\textwidth]{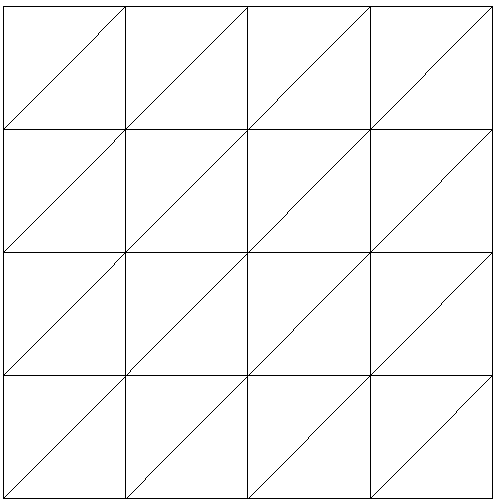} \quad 
		\includegraphics[width=0.15\textwidth]{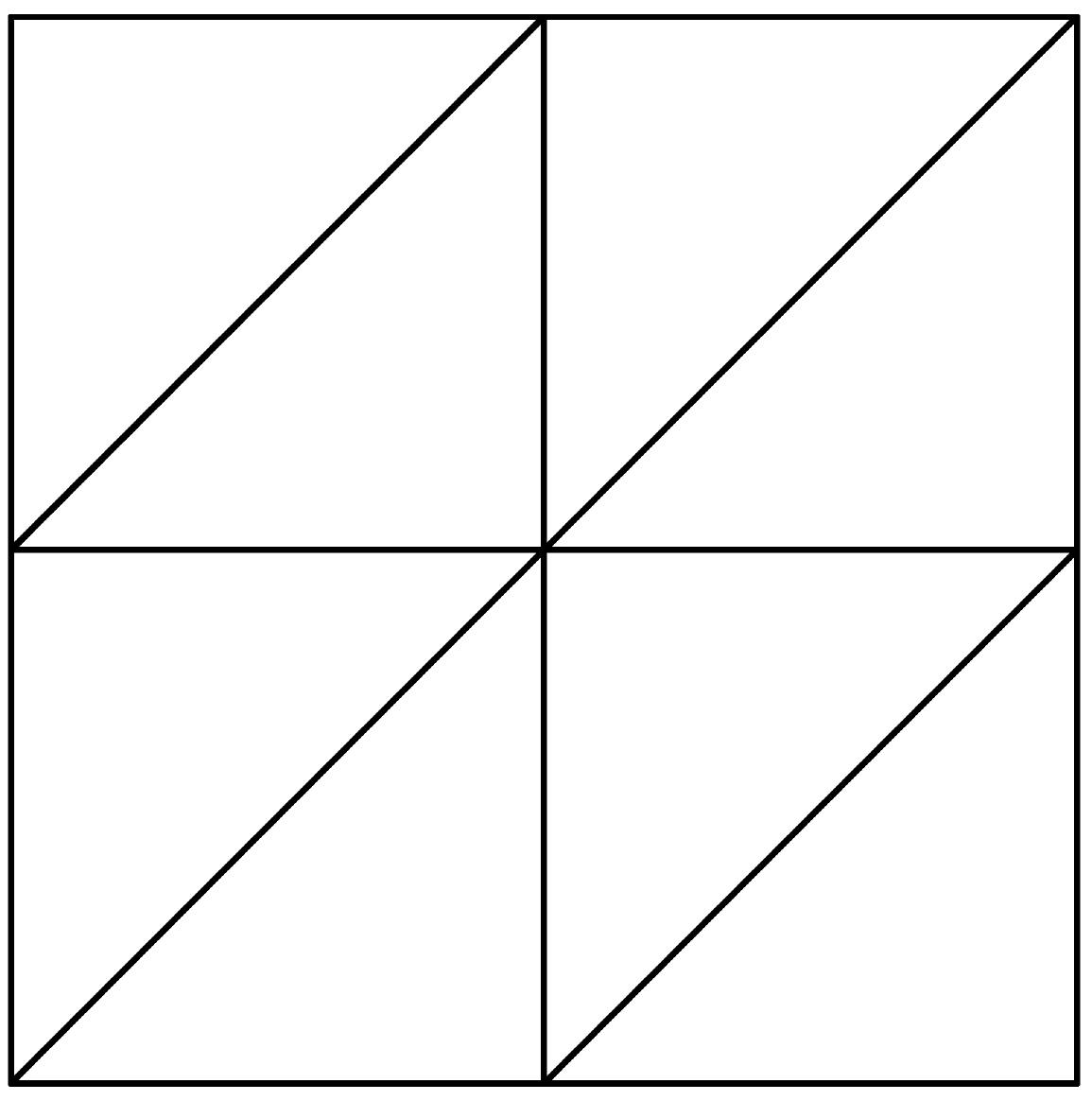} 
	\end{center}
	$$ 
	\mathcal T_1\hskip1in \mathcal T_2\hskip1in \mathcal T_3\hskip1in\mathcal T_4
	$$
	\caption{multilevel grids for piecewise linear functions}
\end{figure}

The grid points of these grids can be given by
$$
x_i^{\ell}=i h_{1,\ell}, y_j^{\ell}=j h_{2,\ell},  i=1, \ldots, m_\ell,
j=1, \ldots, n_\ell.
$$
Here $h_{1,\ell} = 2^{-s + \ell -1}a$ and $h_{2,\ell} = 2^{-t + \ell - 1}b$,
for some $a,b >0$. The above geometric coordinates $(x_i^\ell, y_i^\ell)$
are usually not used in image precess literatures, but they are relevant
in the context of multigrid method for numerical solution of PDEs.
We now consider piecewise linear functions on the sequence of grids
\eqref{grids} and we obtain a nested sequence of linear vector spaces
\begin{equation}
\label{Vk}
\mathcal V_1\supset\mathcal V_2\supset\ldots\supset \mathcal
V_J.
\end{equation}
Here each $\mathcal V_\ell$ consists of all piecewise bilinear (or linear)
functions with respect to the grid \eqref{grids} and \eqref{mn-ell}.
Each $\mathcal V_\ell $ has a set of basis functions:
$\phi_{ij}^\ell\in \mathcal V_\ell$ satisfying:
$$
\phi_{ij}^\ell(x_p,y_q)=\delta_{(i,j), (p,q)} = 
\begin{cases}
1 \quad &\text{if} \quad (p,q) = (i,j), \\
0 \quad &{\text{if}} \quad (p,q)\neq (i,j).
\end{cases}
$$
Thus, for each $v \in \mathcal V_{\ell}$, we have 
\begin{equation}\label{expand}
v(x,y)=\sum_{i,j}v^\ell_{ij}\phi_{ij}^\ell(x,y).
\end{equation}
\subsection{Prolongation}
Given a piecewise (bi-)linear function $\bm v\in\mathcal V_{\ell+1}$, the
nodal values of $\bm v$ on $m_{\ell+1}\times n_{\ell+1}$ grids point
constitute a tensor
$$
v^{\ell+1}\in \mathbb R^{m_{\ell+1}\times n_{\ell+1}}.
$$
We note that $\bm v\in\mathcal V_{\ell}$ thanks to \eqref{Vk} and the
nodal values of $\bm v$ on $\mathcal T_\ell$ 
constitute a tensor
$$
v^{\ell}\in \mathbb R^{m_{\ell}\times n_{\ell}}.
$$
Using the property of piecewise (bi-)linear functions, it is easy to see that 
\begin{equation}
\label{eq:13}
v^{\ell}=\bar P_{\ell+1}^\ell v^{\ell+1},  
\end{equation}
where 
\begin{equation}
\label{mg-prolong}
\bar P_{\ell+1}^\ell: \mathbb R^{m_{\ell+1}\times n_{\ell+1}}\mapsto  \mathbb R^{m_{\ell}\times n_{\ell}},
\end{equation}
which is called a prolongation in multigrid terminology.  More
specifically, 
\begin{equation}
\label{eq:7}
v^{\ell}_{2i-1,2j-1}=  v^{\ell+1}_{i,j},
\end{equation}
with 
\begin{equation}
\label{eq:9}
v^{\ell}_{2i-1, 2j} = \frac{1}{2}(v^{\ell+1}_{i,j} + v^{\ell+1}_{i,j+1}), \quad 
v^{\ell}_{2i, 2j-1} = \frac{1}{2}(v^{\ell+1}_{i,j} + v^{\ell+1}_{i+1,j}),
\end{equation}
and
\begin{equation}
v^{\ell}_{2i, 2j}  =  
\begin{cases}
\frac{1}{4}(v^{\ell+1}_{i,j} + v^{\ell+1}_{i+1,j} + v^{\ell+1}_{i,j+1} + v^{\ell+1}_{i+1,j+1}),  &\text{if $v^\ell$ is piecewise bilinear }, \\
\frac{1}{2}( v^{\ell+1}_{i+1,j} + v^{\ell+1}_{i,j+1}), &\text{if $v^\ell$ is piecewise linear }.
\end{cases}
\end{equation}

\subsection{Pooling, restriction and interpolation}\label{sec:cnn-restriction}
The prolongation given by \eqref{mg-prolong} can be used to transfer
feature from a coarse grid to a fine grid.   
On the other hand, we also
need a mapping, known as restriction,  that transfer data from fine grid to corse grid:
\begin{equation}
\label{mg-restrict}
\bar R_\ell^{\ell+1}: \mathbb R^{m_{\ell}\times n_{\ell}}  \mapsto  \mathbb R^{m_{\ell+1}\times n_{\ell+1}}.
\end{equation}
In multigrid for solving discretized partial differential equation,
the restriction is often taken to be transpose of the prolongation
given by \eqref{mg-prolong}:
\begin{equation}
\label{mg-RP}
\bar R_\ell^{\ell+1} = [\bar P_{\ell+1}^\ell]^T.  
\end{equation}
\begin{lemma}
	If $\tilde P_{\ell+1}^\ell$ takes the form of prolongation
	in multigrid methods for linear finite element functions 
	on the above grids, then $\tilde R_{\ell}^{\ell+1}$ is 
	a convolution with stride $2$ and a $3\times3$
	kernel as:
	\begin{equation}\label{restriction}
	R_\ell^{\ell+1} f=K_R\ast_2 f,
	\end{equation} 
	where, if $\mathcal V_\ell$ is piecewise bilinears, 
	\begin{equation}\label{bi-restrict}
	K_R=
	\begin{pmatrix}
	\frac{1}{4} &\frac{1}{2}&\frac{1}{4}\\
	\frac{1}{2}& 1&\frac{1}{2}\\
	\frac{1}{4}&\frac{1}{2}&  \frac{1}{4} 
	\end{pmatrix},
	\end{equation}
	or, if $\mathcal V_\ell$ is piecewise linears, 
	\begin{equation}
	\label{linear-restrict}
	K_R=
	\begin{pmatrix}
	0 &\frac{1}{2}&\frac{1}{2}\\
	\frac{1}{2}& 1&\frac{1}{2}\\
	\frac{1}{2}&\frac{1}{2}&  0
	\end{pmatrix}.
	\end{equation}
\end{lemma} 
In addition, all these convolutions are applied with zero padding as in \eqref{eq:padding}
, which is consistent with the Neumann boundary condition for applying FEM to 
numerical PDEs. More details will be discussed in \S~\ref{sec:cnn-restriction}.

In the deep learning literature, the restriction such as
\eqref{mg-restrict} is often known as pooling operation.  One popular
pooling is a convolution with stride $s$, with some small integer $s>1$.  

Some other fixed (or untrained) poolings are also often used.  
One popular pooling is the so-called average pooling $R_{avr}$ which
can be a convolution with stride $2$ or bigger using the  kernel $K$
in the form of
\begin{equation}
\label{average-K}
K=
\frac{1}{9}\begin{pmatrix}
1 & 1 & 1 \\
1 & 1 & 1 \\
1 & 1 & 1
\end{pmatrix}. 
\end{equation}
Nonlinear pooling operator is also used, for the example the $(2k+1)
\times (2k+1)$ max-pooling operator with stride $s$ as follows:
\begin{equation}
[{R}_{\rm max}(f)]_{i,j} = \max_{-k\le p ,q\le k} \{f_{s(i-1)+1 + p, s(j-1)+1 +q} \}.
\end{equation}

Another approach to the construction of restriction of pooling can be 
obtained by using interpolation. 
Given 
$$
v^\ell\in \mathbb \mathbb R^{m_{\ell}\times n_{\ell}}, 
$$
let $\bm v\in \mathcal V_\ell$ be the function whose nodal values are
precisely give by $v^\ell$ as in \eqref{expand}.  
Any reasonable linear operator
\begin{equation}
\label{eq:11}
\Pi:  \mathcal V_\ell\mapsto \mathcal V_{\ell+1},
\end{equation}
such as: 
nodal value interpolation, 
Scott-Zhang interpolation and
$L^2$ projection \cite{xu2019FEM},
would give rise to a mapping
\begin{equation}
\label{mg-Pi}
\Pi_\ell^{\ell+1}: \mathbb R^{m_{\ell}\times n_{\ell}}  \mapsto
\mathbb R^{m_{\ell+1}\times n_{\ell+1}},
\end{equation}
such that
$$
v^{\ell+1}=\Pi_\ell^{\ell+1} v^\ell.
$$
As situations permit, we can use these a priori given restrictions to replace
unknown pooling operators to reduce the number of parameters.

\section{Multigrid methods for numerical PDEs}\label{sec:mg}
Let us first briefly describe a geometric multigrid method used to solve the 
following boundary value problem
\begin{equation}
\label{laplace}
-\Delta u = f,  \mbox{ in } \Omega,\quad
\frac{\partial u}{\partial {\bm n}} =0  \mbox{ on } \partial\Omega,\quad
\Omega=(0,1)^2.
\end{equation}

We consider a continuous linear finite element discretization of
\eqref{laplace} on a nested sequence of grids of sizes $n_\ell\times
n_\ell$ with $n_{\ell}=2^{J-\ell+1} + 1$, as shown in the left part of
Fig. \ref{mgrid} and the corresponding sequence of finite
element spaces \eqref{Vk}.

Based on the grid $\mathcal T = \mathcal T_\ell$, the discretized system is
\begin{equation}
\label{laplace-h}
Au=f.
\end{equation}
Here,  $A:\mathbb R^{n\times n}\mapsto \mathbb R^{n\times n}$ is a tensor satisfying
\begin{equation}
\label{uniform-laplace}
(Au)_{i,j}=4u_{i,j}-u_{i+1,j}-u_{i-1,j}-u_{i,j+1}-u_{i,j-1},
\end{equation}
which holds for $1\le i,j \le n$ with zero padding. 
Here we notice that, there exists a $3\times 3$ kernel as
\begin{equation}\label{eq:kernel-A}
K_A = \begin{pmatrix}
0 & -1 & 0 \\
-1 & 4 & -1 \\
0 & -1 & 0
\end{pmatrix},
\end{equation}
with 
\begin{equation}\label{eq:convA}
Au = K_A \ast u.
\end{equation}
Where $\ast$ is the stander convolution operation with zero padding like \eqref{con1}. 
We now briefly describe a simple multigrid method by a mixed use of the terminologies from 
deep learning \cite{goodfellow2017deep} and multigrid methods.

The first main ingredient in GMG is a smoother.  A commonly used smoother is a
damped Jacobi with damped coefficient $\omega$ with $\omega \in (0,2)$,  which can be written as $S_{0}:\mathbb R^{n\times n}\mapsto
\mathbb R^{n\times n}$ satisfying
\begin{equation}
\label{jacobi1}
(S_{0}f)_{i,j}={\omega\over 4}f_{i,j},
\end{equation}
for equation \eqref{laplace-h} with initial guess zero.
If we apply the Jacobian iteration twice, then
$$
S_1(f) = S_{0} f + S(f - A(S_{0}f)),
$$
with element-wise form
\begin{equation} 
\begin{aligned}
\label{jacobi2}
[S_1(f)]_{i,j} &={1\over 4}\omega(2-\omega)f_{i,j} + {\omega^2\over 16}(f_{i+1,j}+f_{i-1,j}+f_{i,j+1}+f_{i,j-1}).
\end{aligned}
\end{equation}
Then we have 
\begin{equation}\label{eq:kernel-S}
K_{S_{0}} = {\omega \over 4},
\end{equation}
and 
\begin{equation}\label{eq:kernel-S2}
K_{S_1} = \begin{pmatrix}
0 & \frac{\omega^2}{16} & 0 \\
\frac{\omega^2}{16} & {\omega(2-\omega) \over 4} & \frac{\omega^2}{16}  \\
0 & \frac{\omega^2}{16}  & 0
\end{pmatrix},
\end{equation}
such that 
\begin{equation}\label{eq:convS}
S_{0}f = K_{S_{0}} \ast f \quad S_1 f = K_{S_1} \ast f.
\end{equation}
Similarly, we can define 
$S^{\ell}: \mathbb{R}^{n_\ell \times n_\ell} \mapsto \mathbb{R}^{n_\ell \times n_\ell}$.

We use prolongation $P_{\ell+1}^\ell: R^{n_{\ell+1}\times n_{\ell+1}}\mapsto R^{n_{\ell}\times n_{\ell}}$
as defined in \eqref{mg-prolong} and restriction $R_{\ell}^{\ell+1} = (P_{\ell+1}^{\ell})^T$. Further more,
we use the following relationship to define coarse operation
\begin{equation}\label{eq:def_coarse}
A^{\ell+1}=R_{\ell}^{\ell+1} A^{\ell}P_{\ell+1}^{\ell} \quad (\ell = 1:J-1),
\end{equation}
with $A^1 = A$.


Using the smoother $S^\ell$, prolongation $P^{\ell}_{\ell+1}$, restriction $R_{\ell}^{\ell+1}$ and mapping
$A^\ell$ as given in \eqref{eq:def_coarse}, we can formulate the following algorithm
as a major component of a multigrid algorithm.
\begin{breakablealgorithm}
	\caption{$(u^{\ell,\nu_\ell}: ~\ell = 1:J) = {\text{MG0}}(f; J,\nu_1, \cdots, \nu_J)$}
	\label{alg:L-Slash0}
	\begin{algorithmic}
		\State Set up
		$$
		f^1 = f, \quad u^{1,0}=0.
		$$
		\State Smoothing and restriction from fine to coarse level (nested)
		\For{$\ell = 1:J$}
		\State Pre-smoothing:
		\For{$i = 1:\nu_\ell$}
		\State
		\begin{equation}\label{eq:smoothing}
		u^{\ell,i} = u^{\ell,i-1} + S^{\ell} (f^\ell - A^\ell u^{\ell,i-1}).
		\end{equation}
		\EndFor
		\State Form restricted residual and set initial guess:
		$$
		u^{\ell+1,0} = 0, \quad f^{\ell+1} = R^{\ell+1}_\ell(f^\ell - A^{\ell} u^{\ell,\nu_\ell}).
		$$
		\EndFor
	\end{algorithmic}
\end{breakablealgorithm}

Using the above algorithm, there are different multigrid algorithms such as: $\backslash$-cycle, V-cycle and W-cycle.
Let us now only give one special form of multigrid algorithm as follows.
\begin{breakablealgorithm}
	\caption{$u = {\backslash\text{-MG}}(f; J,\nu_1, \cdots, \nu_J)$}
	\label{alg:L-Slash1}
	\begin{algorithmic}
		\State Call Algorithm \ref{alg:L-Slash0},
		$$
		(u^{\ell,\nu_\ell}: ~\ell = 1:J) = {\text{MG0}}(f; J,\nu_1, \cdots, \nu_J).
		$$
		\State Prolongation and restriction from coarse to fine level
		\For{$\ell = J-1:1$}
		\State Coarse grid correction (residual)
		\begin{equation}\label{eq:coarse-correstion}
		u^{\ell,\nu_\ell} \leftarrow u^{\ell,\nu_\ell} + P_{\ell+1}^{\ell}u^{\ell+1, \nu_{\ell+1}}.
		\end{equation}
		\EndFor
		\State Output
		$$
		u = u^{1,\nu_1}.
		$$
	\end{algorithmic}
\end{breakablealgorithm}

\section{MgNet: a new network structure}\label{sec:mgnet}
In this section, we introduce a new neural network structure,
named as MgNet, motivated by the multigrid algorithm, 
Algorithm \ref{alg:L-Slash0}, as discussed in the previous section.


First, given the data-feature equation \eqref{Auf}, we consider
its restrictions to grid $\ell$ as follows:
\begin{equation}
\label{Auf-ell}
A^\ell(u^\ell) = f^\ell, \quad \ell=1:J,
\end{equation}
where
\begin{equation}
\label{f-ell}
f^{\ell}\in\mathbb R^{m_\ell\times n_\ell\times c_{f,\ell}},
\end{equation}
and 
\begin{equation}
\label{u-ell}
u^{\ell}\in\mathbb R^{m_\ell\times n_\ell\times c_{u,\ell}}.
\end{equation}
We are now in a position to state the main algorithm, namely
MgNet as:
\begin{breakablealgorithm}
	\caption{$u^J={\rm MgNet}(f; J,\nu_1, \cdots, \nu_J)$}
	\label{alg:mgnet}
	\begin{algorithmic}
		\State Initialization:  $f^1 = f_{\rm in}(f)$, $u^{1,0}=0$
		\For{$\ell = 1:J$}
		\For{$i = 1:\nu_\ell$}
		\State Feature extraction (smoothing):
		\begin{equation}\label{mgnet}
		u^{\ell,i} = u^{\ell,i-1} + B^{\ell,i}  \left({f^\ell -  A^{\ell} (u^{\ell,i-1})}\right).
		\end{equation}
		\EndFor
		\State Note: 
		$
		u^\ell= u^{\ell,\nu_\ell} 
		$
		\State Interpolation and restriction:
		\begin{equation}
		\label{interpolation}
		u^{\ell+1,0} = \Pi_\ell^{\ell+1}u^{\ell}.
		\end{equation}
		\begin{equation}
		\label{restrict-f}
		f^{\ell+1} = R^{\ell+1}_\ell(f^\ell - A^\ell(u^{\ell})) + A^{\ell+1} (u^{\ell+1,0}).
		\end{equation}
		\EndFor
	\end{algorithmic}
\end{breakablealgorithm}

Here, $f_{\rm in}(\cdot)$ is the data initialization process as a usual step in many classical CNNs 
\cite{krizhevsky2012imagenet,he2016deep,he2016identity,huang2017densely}.
It may depend on different data sets and problems, we will discuss it later in \S~\ref{sec:ini-mgnet} and \S~\ref{sec:CNNs}.
For the main structure, the next diagram gives a brief illustration for the schema of
MgNet as shown in Algorithm \ref{alg:mgnet} with \eqref{linearA} and \eqref{extractor}.\\
\begin{figure}[H]
	\begin{center}
		\includegraphics[width=.4\textwidth, height=.32\textheight]{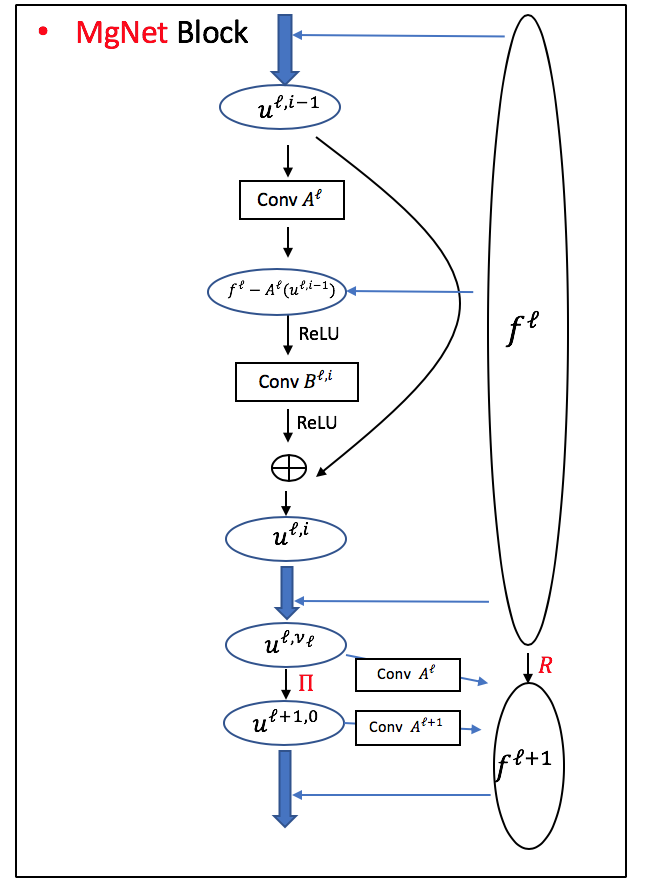} 
	\end{center}
	\caption{Structure of MgNet}
	\label{fig:mgnet}
\end{figure}

Here we may have some more general MgNet structures by replacing
the feature extraction (smoothing) step \eqref{mgnet} with some 
other iterative schemes such as:
\begin{description}
	\item[(Single step) MgNet] 
	\begin{equation}\label{single-step-mgnet}
	u^{\ell,i} = u^{\ell,i-1} + B^{\ell,i}  \left({f^\ell -  A^{\ell} (u^{\ell,i-1})}\right), \quad i = 1:\nu_\ell.
	\end{equation}
	\item[Multi-step MgNet]
	\begin{equation}\label{multi-step-mgnet}
	u^{\ell,i} = \sum_{j=0}^{i-1} \alpha^{\ell,i}_j \left( u^{\ell,j} + B^{\ell,i}_j  ( {f^\ell -  A^{\ell} (u^{\ell,j})} ) \right), \quad i = 1:\nu_\ell.
	\end{equation}
	\item[Chebyshev-semi MgNet]
	\begin{equation}\label{Chebyshev-semi-mgnet}
	u^{\ell, i} = \omega^{\ell,i}\left(u^{\ell,i-1} + B^{\ell,i}\left(f^\ell - A^\ell(u^{\ell,i-1})\right)\right)+ (1- \omega^{\ell,i}) u^{\ell,i-2}, \quad i = 1:\nu_\ell.
	\end{equation}
\end{description}
Where $B^{\ell,i}$ and $ B^{\ell,i}_j$ can be some appropriate nonlinear forms such as 
$\eqref{extractor}$ in the basic MgNet in Algorithm \ref{alg:mgnet}  which 
can relate to iResNet model naturally.
Roughly speaking, muli-step MgNet structure and Chebyshev-semi MgNet may be related to DenseNet \cite{huang2017densely}
and LM-ResNet \cite{lu2018beyond} with a special choice of the nonlinear form of $ B^{\ell,i}_j$ and $B^{\ell,i}$.

Let us focus on the basic MgNet form in Algorithms \ref{alg:mgnet}, 
the first important property of MgNet is that it recovers the fine to coarse process of multigrid methods
as in  Algorithm \ref{alg:L-Slash0}.
\begin{theorem}
	If $A^\ell$, $R_\ell^{\ell+1}$ and $B^{\ell,i} = S^{\ell}$ are all linear operations as described in multigrid method
	in \S~\ref{sec:mg}. 
	Then Algorithm \ref{alg:L-Slash0} is equivalent to Algorithm \ref{alg:mgnet} with any choice of $\Pi_\ell^{\ell+1}$.
\end{theorem}
\begin{proof}
	Here we replace $u^{\ell,i}$ and $f^{\ell}$ by $\tilde u^{\ell,i}$ and $\tilde f^{\ell}$ in MgNet. 
	What we want to prove are
	\begin{equation}\label{eq:f-u}
	\tilde f^{\ell} =  f^{\ell} + A_\ell \tilde u^{\ell,0} \quad \text{and} \quad u^{\ell,i} = \tilde u^{\ell, i} - \tilde u^{\ell, 0},
	\end{equation}
	with $u^{\ell,i}$, $f^\ell$ in Algorithm \ref{alg:L-Slash0} and 
	$\tilde u^{\ell,i}$, $\tilde f^\ell$ in Algorithm \ref{alg:mgnet} for any choice of $\Pi_\ell^{\ell+1}$. 
	We prove this result by induction. 
	\begin{itemize}
		\item It is easy to check that $\ell = 1$ is right by taking $\theta  = \rm{id}$. 
		\item Once the above equation \eqref{eq:f-u} is right for $\ell$, 
		let us prove the corresponded result for $\ell+1$.
		\begin{itemize}
			\item For $\tilde f^{\ell+1}$, as the definition in Algorithm \ref{alg:mgnet}, we have
			\begin{equation}
			\begin{aligned}
			\tilde f^{\ell+1} &= R_\ell^{\ell+1}(\tilde f^\ell - A^\ell \tilde u^{\ell,\nu_\ell}) + A^{\ell+1}\tilde u^{\ell+1,0} \\
			&= R_\ell^{\ell+1}( f^\ell + A^{\ell} \tilde u^{\ell,0}- A^{\ell} \tilde u^{\ell,\nu_\ell}) + A^{\ell+1}\tilde u^{\ell+1,0} \\
			&= R_\ell^{\ell+1}( f^\ell - A^{\ell} (\tilde u^{\ell,\nu_\ell}- u^{\ell,0})) + A^{\ell+1}\tilde u^{\ell+1,0} \\
			&= R_\ell^{\ell+1}(f^\ell - A^{\ell} u^{\ell,\nu_\ell}) + A^{\ell+1}\tilde u^{\ell+1,0} \\
			&=  f^{\ell+1} + A^{\ell+1}\tilde u^{\ell+1,0}.
			\end{aligned}
			\end{equation}
			\item For $u^{\ell+1,i}$, first we have 
			\begin{equation}
			u^{\ell+1,0} = 0 =\tilde u^{\ell+1, 0} - \tilde u^{\ell+1, 0},
			\end{equation}
			then we prove 
			\begin{equation}\label{u:i+1}
			u^{\ell+1,i} = \tilde u^{\ell+1, i} - \tilde u^{\ell+1, 0}
			\end{equation} by induction for $i$.
			
			We assume \eqref{u:i+1} holds for $0,1,\cdots,i-1$. Let us miner $\tilde u^{\ell+1, 0}$ in both sides of 
			the smoothing process \eqref{mgnet} in Algorithm \ref{alg:mgnet}. Then we have
			\begin{equation}
			\begin{aligned}
			\tilde u^{\ell+1,i} - \tilde u^{\ell+1, 0} &= \tilde u^{\ell+1,i-1} - \tilde u^{\ell+1, 0} + B^{\ell+1,i} (\tilde f^{\ell+1} - A^{\ell+1} \tilde u^{\ell+1,i-1}) \\
			&= \tilde u^{\ell+1,i-1} - \tilde u^{\ell+1, 0} + B^{\ell+1,i} (f^{\ell+1} + A^{\ell+1}\tilde u^{\ell+1,0} - A^{\ell+1} \tilde u^{\ell+1,i-1} )\\ 
			&= u^{\ell+1,i-1} + B^{\ell+1,i} (f^{\ell+1} - A^{\ell+1}u^{\ell+1,i-1} ).
			\end{aligned}
			\end{equation}
			This is exact the smoothing process in Algorithm \ref{alg:L-Slash0} as we take $ B^{\ell+1,i} = S^{\ell+1}$.
		\end{itemize}
	\end{itemize}
\end{proof}

Similar to Algorithm \ref{alg:L-Slash1} in $\backslash$-MG or 
the corresponding version in V-cycle multigrid, there exists a related 
V-MgNet that includes a process from coarse to fine grids. 
\begin{breakablealgorithm}
	\caption{$u^1 =\text{V-}{\rm MgNet}(f; J,\nu_1, \cdots, \nu_J; \nu'_1, \cdots, \nu'_J )$}
	\label{alg:mgnet1}
	\begin{algorithmic}
		\State 
		$$
		(\bar u^{1,0}, \bar u^1, f^1, \bar u^{2,0}, \bar u^2, f^2,\cdots, \bar u^{J,0},\bar u^J, f^J) = \text{MgNet}(f; J,\nu_1, \cdots, \nu_J).
		$$
		\For{$\ell = J-1 : 1$}
		\begin{equation}\label{V-prolongation}
		u^{\ell,0} \leftarrow \bar u^{\ell} + P_{\ell+1}^{\ell} (u^{\ell+1} - \bar u^{\ell+1,0}).
		\end{equation}
		\For{$i = 1:\nu'_\ell$}
		\State 
		\begin{equation}\label{V-mgnet}
		u^{\ell,i} \leftarrow u^{\ell,i-1} + B'_{\ell,i}  ({f^\ell -  A^{\ell} (u^{\ell,i-1})}).
		\end{equation}
		\EndFor
		\State  
		$$
		u^{\ell} \leftarrow u^{\ell,\nu_\ell'} .
		$$
		\EndFor
	\end{algorithmic}
\end{breakablealgorithm}
This type of V-MgNet
makes use of prolongation operators that correspond directly to the co-called
deconvolution operations in CNN models \cite{noh2015learning}. In addition, the 
correction steps such as \eqref{V-prolongation} correspond directly to 
the symmetric skip connection in many autoencoder type models such as U-net \cite{ronneberger2015u}
and others \cite{mao2016image, liu2017when, lin2017feature}.  Furthermore, we can actually recover 
these U-net type CNN models from V-MgNet with similar situation as in MgNet and iResNet which 
we will discuss later in \S~\ref{sec:relation}.

Despite of the simplicity look of Algorithm \ref{alg:mgnet}, there
are rich mathematical structures and variants which we briefly discuss below.

\subsection{Initialization: feature space channels}\label{sec:ini-mgnet}
Initially for $\ell=1$,  we take $m_1 = m$ and $n_1 = n$ and we may define the linear mapping 
\begin{equation}
\label{eq:6}
\theta: \mathbb R^{m\times n\times c}
\mapsto \mathbb R^{m_1\times n_1\times c_1},
\end{equation}
to obtain $f^{1} = f_{\rm in}(f) = \theta(f)$ with $c$ given in \eqref{data-c} changed to the channel of the initial
data space to $c_1$.   Usually
\begin{equation}
\label{cc}
c_1\ge c.  
\end{equation}
One possibility is that we choose $c_1=c$.  In this case, we choose
$\theta=$identity.   But in general, we may need to choose $c_1\gg
c$. One possible advantage of preprocessing the RGB ($c=3$) to 
different color spaces is that we can better choose what kind of
features the CNN can detect, and under what 
conditions those detections will be invariant.

One possibility of understanding and modifying this step 
is to decompose the data $f$ into a number of more
specialized data
\begin{equation}
\label{decomp-f}
f=\sum_{k=1}^{c_1}\xi_kf^1_k  =\xi^Tf^1.
\end{equation}
We may use some knowledge from image processing or physics to
design a procedure to obtain the right decomposition of
\eqref{decomp-f}, or we can just train it. 
Conceivably, we may view $f^{1} = \theta(f)$ as a special approximation solution of
\eqref{decomp-f} with the same sparsity pattern to $\xi$. 

\subsection{Extracted Units: $u^{\ell}$ and channels}
The first new feature and the main new ingredient 
in the proposed neural network is the introduction 
of feature variables  $u^{\ell}$ in \eqref{u-ell}, which will be known
as the extracted units. 

One main ingredient in our MgNet in addition to the data variables 
is the introduction of feature variables $u^{\ell}$ in \eqref{u-ell}, 
known as the extracted-units. 
The so-called dual path networks (DPN) model in \cite{chen2017dual} also
makes use of additional variables. 
DPN is a special CNN obtained by combining two different CNN models such as
ResNet and DenseNet. 
If we view $u^{\ell,i}$ and $f^\ell$ as two different paths, MgNet can
be related to DPN model.
We note that, $u^{\ell,i}$ and $f^\ell$ communicate
to each other with a special version as in \eqref{mgnet} 
with a special restriction form as in \eqref{restrict-f}. 
We can recover DPN from MgNet by using two different
smoothing processes and combining them. 

We emphasize that the extracted-units $u^{\ell,i}$ and the data $f^\ell$ can have
different numbers of channels:
\begin{equation}
\label{uf-channels}
u^{\ell,i}\in \mathbb{R}^{m_\ell\times n_\ell \times c_{u,\ell}}, \quad
f^\ell\in \mathbb{R}^{m_\ell\times n_\ell \times c_{f,\ell} }.
\end{equation}
One possibility is that the number of channels for both $u$ and $f$ remain
unchanged in different grids:  
\begin{equation}
\label{cfl}
c_{f,\ell}=c_f, \quad \ell=1:J,   
\end{equation}
and 
\begin{equation}
\label{ufl}
c_{u,\ell}=c_{u}, \quad \ell=1:J.   
\end{equation}
Both $c_f$ and $c_{u}$ are two super-parameters that need to be tuned, 
and we may even take $c_u = c_f$.

\subsection{Poolings: $\Pi_{\ell}^{\ell+1}$ and $R_{\ell}^{\ell+1}$}
The pooling $\Pi_{\ell}^{\ell+1}$ in \eqref{restriction} and
$R_{\ell}^{\ell+1}$ in \eqref{restrict-f} are in general different.
They can be trained in general, but they may be a priori chosen.

There are many different possibilities to choose $\Pi_{\ell}^{\ell+1}$. 
The simplest choice of $\Pi_{\ell}^{\ell+1}$ is 
\begin{equation}
\label{eq:8}
\Pi_{\ell}^{\ell+1}=0.
\end{equation}
A more sophisticated choice can be obtained by considering an
interpolation from fine grid to coarse (that, for example preserves linear function
locally).  Namely
\begin{equation}
\label{Pi}
\Pi_{\ell}^{\ell+1}=\bar\Pi_{\ell}^{\ell+1}\otimes I_{c_\ell\times c_\ell} ,
\end{equation}
with $\bar\Pi_{\ell}^{\ell+1}$ given by~\eqref{mg-Pi}.
It can be implemented by group convolution \cite{zhang2017interleaved} 
with channels as groups number.

\subsection{Data-feature mapping: $A^{\ell}$}
The second new feature of MgNet is that this data-feature mapping
only depends on the grid ${\cal T}_\ell$, and it does not depend on layers
within the same grid.  This amounts to a significant saving of the number of
parameters especially for deep ResNet models.  In comparison, the existing CNN, such as iResNet, can be
interpreted as a network related to the case that $A^{\ell}$ is
replaced by $A^{\ell, i}$, namely
\begin{equation}\label{u-resnet}
u^{\ell,i} = u^{\ell,i-1} + B^{\ell,i}  (f^\ell -  A^{\ell,i} (u^{\ell,i-1}) ),
\end{equation}
which will be discussed later in \S~\ref{sec:relation}.

The data-feature mapping: $A^{\ell}$ can be either linear
\eqref{linearA}, or nonlinear \eqref{nonlinearA}.  The underlying
convolution kernels can be different on different grids and they can
all be trained.

\subsection{Feature extractors: $B^{\ell,i}$}
There are some freedoms in choosing these feature extrators.  
One common choice of extractors is given by \eqref{extractor}, namely 
\begin{equation}
\label{extractor-ell}
B^{\ell,i}=\sigma\circ \eta^{\ell,i}\circ\sigma.
\end{equation}

Other than the level dependent extractors, the following 
different strategies can be used
\begin{description}
	\item[Constant Extractors]: $B^{\ell,i}=B^{\ell}$ for   $i=1:\nu_\ell$.
	\item[Scaled Extractors]: $B^{\ell,i}=\alpha_iB^{\ell}$ for   $i=1:\nu_\ell$.
	\item[Variable Extractors]: $B^{\ell,i}$.
\end{description}

This brief framework gives us the basic principle on designing 
a CNN models for classification. All models are seen as the special
choice of data-feature mapping $A^\ell$, feature extractors $B^{\ell,i}$ 
and the pooling operators $\Pi_{\ell}^{\ell+1}$ with $R_{\ell}^{\ell+1}$.

\section{Some classic CNN models}\label{sec:CNNs}
In this section, we will use the notation introduced above to 
give a brief description of some classic CNN models.

\subsection{LeNet-5, AlexNet and VGG}
The  LeNet-5 \cite{lecun1998gradient}, AlexNet \cite{krizhevsky2012imagenet} and VGG \cite{simonyan2014very}
can be written as:
\begin{equation}
\begin{cases}
f^{1,0} &= \theta^0(f), \\
\text{\bf For} &\ell = 1:J \\
\quad &\text{\bf For} \quad i = 1:\nu_\ell \\
&f^{\ell,i} = \theta^{\ell,i} \circ \sigma (f^{\ell, j-1}), \\
\quad &\text{\bf EndFor} \\
\quad \quad f^{\ell+1,0} &= R_\ell^{\ell+1}( f^{\ell,m+\ell}), \\
\text{\bf EndFor} &\\
\end{cases}
\end{equation}
where $R_\ell^{\ell+1}$ can be general pooling operators and $\theta^{\ell,i}$ can be convolution with stride 1, 
or fully connected operators.  
Then the CNN model will be defined by
\begin{equation}\label{eq:cnndefine}
H_0(f) = f^{J,\nu_J}.
\end{equation}
In these three classic CNN models, they still need some 
extra fully connected layers after $H_0(f)$ but before the logistic regression \eqref{eq:log_reg}. 
These fully connected layers are removed in ResNet to be described below.
\subsection{ResNet}
The ResNet \cite{he2016deep} can be written as
\begin{equation}\label{ori-ResNet}
\begin{cases}
f^{1,0} &=f_{\rm in}(f), \\
\text{\bf For} &\ell = 1:J \\
\quad &\text{\bf For} \quad i = 1:\nu_\ell \\
&f^{\ell,i} = \sigma \left( f^{\ell, i-1} + \mathcal{F}^{\ell, i} (f^{\ell,i-1}) \right), \\
\quad &\text{\bf EndFor} \\
\quad \quad f^{\ell+1,0} &= \sigma \left( R_\ell^{\ell+1} (f^{\ell, \nu_\ell} )+ \mathcal{F}^{\ell, 0} (f^{\ell, \nu_\ell} ) \right), \\
\text{\bf EndFor} &\\
H_0(f) &=  R_{\rm ave}(f^{L,\nu_\ell}). \\
\end{cases}
\end{equation}
Here $f_{\rm in}(\cdot)$ may depend on different data set and problems 
such as $f_{\rm in}(f) = \sigma \circ \theta^0(f)$ for CIFAR \cite{krizhevsky2009learning} and
$f_{\rm in}(f) = R_{\rm max}\circ \sigma \circ \theta^0(f)$ for ImageNet \cite{deng2009imagenet} as in \cite{he2016deep}.
In addition $\sigma \left( f^{\ell, i-1} + \mathcal{F}^{\ell, i} (f^{\ell,i-1}) \right)$ is often called the basic ResNet block with
$$
\mathcal{F}^{\ell,i} (f^{i-1}) = \xi^{i} \circ \sigma \circ \eta^{i} (f^{i-1}).
$$
Generally, $\xi^{\ell,i}$ and $\eta^{\ell,i}$ takes the form of \label{eq:conv-1} with zero padding and stride 1,
except, $\eta^{\ell,0}$  is taken as convolution with stride 2 with the same output dimension of $R_\ell^{\ell+1}$.

\subsection{iResNet} 
The iResNet \cite{he2016identity} can be written as:
\begin{equation}\label{eq:iResNet1}
\begin{cases}
f^{1,0} &= f_{\rm in}(f), \\
\text{\bf For} &\ell = 1:J \\
\quad &\text{\bf For} \quad i = 1:\nu_\ell \\
&f^{\ell,i} = f^{\ell, i-1} + \mathcal{F}^{\ell, i} (f^{\ell,i-1}), \\
\quad &\text{\bf EndFor} \\
\quad \quad f^{\ell+1,0} &=  R_\ell^{\ell+1} (f^{\ell, \nu_\ell} )+ \mathcal{F}^{\ell, 0} (f^{\ell, \nu_\ell} ) , \\
\text{\bf EndFor} &\\
H_0(f) &=  R_{\rm ave}(f^{L,\nu_\ell}), \\
\end{cases}
\end{equation}
where $f_{\rm in}(\cdot)$ shares the same setup with ResNet but 
$$
\mathcal{F}^{\ell,i} (f^{\ell,i -1}) = \xi^{\ell,i} \circ \sigma \circ \eta^{\ell,i} \sigma (f^{\ell,i-1}).
$$
The only difference between ResNet and iResNet can be viewed as 
putting a $\sigma$ in different places. 

\subsection{DenseNet}
The DenseNet \cite{huang2017densely} model can be written as:
\begin{equation}\label{eq:densenet1}
\begin{cases}
f^{1,0} &= f_{\rm in}(f), \\
\text{\bf For} &\ell = 1:J \\
\quad &\text{\bf For} \quad i = 1:\nu_\ell \\
&f^{\ell,i} = \sigma \left( \sum_{j=0}^{i-1} [\theta^{\ell,i}]_{j} \ast f^{\ell,j} \right) ,\\
\quad &\text{\bf EndFor} \\
\quad \quad f^{\ell+1,0} &=  R_\ell^{\ell+1} ([f^{\ell,0,}, \cdots, f^{\ell,\nu_\ell}] ) , \\
\text{\bf EndFor} &\\
H_0(f) &=  R_{\rm ave}(f^{L,\nu_\ell}). \\
\end{cases}
\end{equation}
Here $[f^{\ell,0}, \cdots, f^{\ell,i}]$ represents the collection of 
all the previous output in $\ell$-th grids after $i$-th smoother in the channel dimension,
and
\begin{equation}
\theta^{\ell,i} = \left( [\theta^{\ell,i}]_{0}, \cdots,   [\theta^{\ell,i}]_{i-1}\right): 
\mathbb{R}^{m_\ell\times n_\ell \times (\sum_{j=0}^{i-1}k_j )} \mapsto  \mathbb{R}^{m_\ell\times n_\ell \times k_i},
\end{equation}
where $[\theta^{\ell,i}]_{j}: \mathbb{R}^{m_\ell\times n_\ell \times k_j} \mapsto  \mathbb{R}^{m_\ell\times n_\ell \times k_i}$ for $j = 0:i-1$.
Roughly speaking, the main iterative step in DenseNet is almost the same as the semi-iterative iterative 
process \eqref{eq:multi} if we ignore the nonlinear activation function $\sigma$ and the fix the channel dimension $k_j$.

In our paper, we mainly consider the connection between MgNet and ResNet type models from the viewpoint of 
single step (residual correction) iterative scheme. In addition, we also make some discussion about the
relationship between Multi-step MgNet and DenseNet using the idea of multi-iterative method.

The development of the first three models is often shown with next diagrams:
\begin{figure}[H]
	\begin{center}
		\includegraphics[width=.6\textwidth, height=.13\textheight]{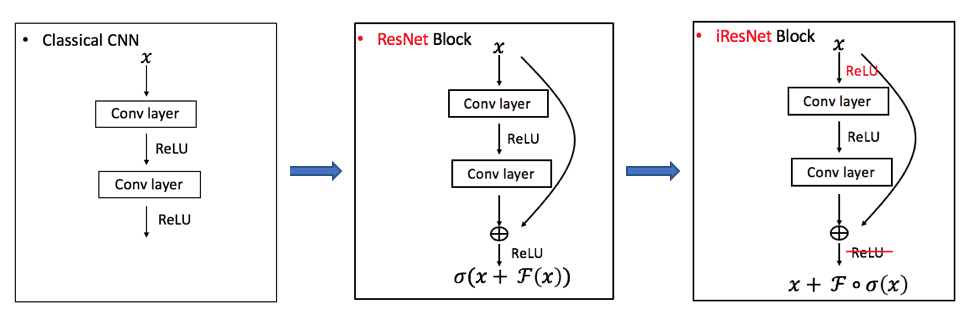} 
	\end{center}
	\caption{Comparison of CNN Structures}
\end{figure}

Without loss of generality, we extract the key 
feedforward steps on the same grid in different CNN models as follows.
\begin{description}
	\item[Classic CNN] 
	\begin{equation}\label{eq:cCNN}
	f^{\ell,i} = \xi^i \circ \sigma (f^{\ell,i-1}) \quad \text{or} \quad f^{\ell,i} = \sigma \circ \xi^{i} (f^{\ell,i-1}) .
	\end{equation}
	\item[ResNet] 
	\begin{equation}\label{eq:ResNet}
	f^{\ell,i} = \sigma( f^{\ell,i-1} + \xi^{\ell,i} \circ \sigma \circ \eta^{\ell,i}(f^{\ell,i-1})).
	\end{equation}
	\item[iResNet]
	\begin{equation}\label{eq:iResNet}
	f^{\ell,i} = f^{\ell,i-1} + \xi^{\ell,i} \circ \sigma \circ \eta^{\ell,i}\circ \sigma(f^{\ell,i-1}).
	\end{equation}
	\item[DenseNet]
	\begin{equation}\label{eq:DenseNet}
	f^{\ell,i} = \sigma \left( \sum_{j=0}^{i-1} [\theta^{\ell,i}]_{j} \ast f^{\ell,j} \right).
	\end{equation}
	
\end{description}

\section{Variants and generalizations of MgNet}\label{sec:relation}

The MgNet model algorithm is one very basic and it can be generalized
in many different ways. It can also be used as a guidance to modify and 
extend many existing CNN models. 

The following result show how MgNet is related to he iResNet \cite{he2016identity}. 
\begin{theorem}\label{thm:mgnet1}
	The MgNet model Algorithm \ref{alg:mgnet}, 
	with $A=\xi^\ell$ and $B^{\ell,i}=\sigma \circ \eta^{\ell,i}\circ\sigma$, 
	admits the following identities
	\begin{equation}\label{dualmgnet}
	f^{\ell, i} = f^{\ell, i-1} -  \xi^{\ell} \circ \sigma \circ \eta^{\ell,i}\circ \sigma (f^{\ell,i-1}), \quad i = 1:\nu_\ell, \\
	\end{equation}
	where
	\begin{equation}
	\label{eq:5}
	f^{\ell,i} = f^{\ell} - \xi^{\ell} (u^{\ell,i}).   
	\end{equation}
	Furthermore, \eqref{dualmgnet} represents iResNet~\cite{he2016identity} 
	as shown in \eqref{eq:iResNet}.
\end{theorem}

\begin{proof}
	Because of the linearity of $\xi^\ell$ and invariant within the same grid $\ell$, 
	we can apply $\xi^\ell$ on both sides of \eqref{mgnet} and minus with
	$f^\ell$, thus we have
	$$
	f^{\ell} - \xi^{\ell} (u^{\ell,i})= f^{\ell} - \xi^\ell(u^{\ell,i-1}) -
	\xi^{\ell} \circ \sigma \circ \eta^{\ell,i}\circ \sigma (f^\ell - \xi^\ell(u^{\ell,i-1})).
	$$
	This finish the proof with definition in \eqref{eq:5}.
\end{proof}

The above result is very simple but critically important.
In view of Theorem \ref{thm:mgnet1}, it shows how multigrid and 
CNN are intimately related. Furthermore, it provides a different version
of iResNet, which can be viewed as the dual version of the original iResNet.
This relation is quit similar with the dual relation of $u$ and $f$
in multigrid method \cite{xu2017algebraic}.
\begin{lemma}\label{thm:mgnet2} 
	The ResNet~\cite{he2016deep} step
	as in \eqref{eq:ResNet} 
	admits the following relation:
	\begin{equation}\label{tilde-resnet}
	\tilde f^{\ell,i} =\sigma(\tilde f^{\ell,i-1}) -
	\xi^{\ell,i} \circ \sigma \circ \eta^{\ell,i}\circ \sigma( \tilde f^{\ell,i-1}),
	\end{equation}
	where
	\begin{equation}\label{tilde-f}
	\tilde f^{\ell,i} = f^{\ell, i-1} -\xi^{\ell,i} \circ \sigma \circ \eta^{\ell,i} (f^{\ell,i-1}).
	\end{equation}
\end{lemma}
\begin{proof}
	First, we apply $ \xi^{\ell,i+1} \circ \sigma \circ \eta^{\ell,i+1}$ 
	on the both sides of \eqref{eq:ResNet} and get
	\begin{equation}\label{resnet1}
	\xi^{\ell,i+1} \circ \sigma \circ \eta^{\ell,i+1}( f^{\ell,i} ) = 
	\xi^{\ell,i+1} \circ \sigma \circ \eta^{\ell,i+1}\circ \sigma( \tilde f^{\ell,i} ).
	\end{equation}
	Minus by $f^{\ell,i}$ on the both sides and recall the definition in \eqref{tilde-f}, we have
	\begin{equation*}
	\tilde f^{\ell,i+1} = f^{\ell,i} - \xi^{\ell,i+1} \circ \sigma \circ \eta^{\ell,i+1}\circ \sigma( \tilde f^{\ell,i}).
	\end{equation*}
	By the definition of $f^{\ell,i} = \sigma(\tilde f^{\ell,i})$, we finish this proof.
\end{proof}

We call the above form \eqref{tilde-resnet} as
$\sigma$-ResNet, similar to the MgNet we replace $\xi^{\ell,i}$ by $\xi^{\ell}$  and get 
the next Mg-ResNet form as:
\begin{equation}\label{mg-resnet}
f^{\ell,i} =\sigma(f^{\ell,i-1}) -
\xi^{\ell} \circ \sigma \circ \eta^{\ell,i}\circ \sigma(f^{\ell,i-1}).
\end{equation}

If we take these pooling and prolongation operators
as discussed in the previous sections and focus on 
the iterative forms on a certain grid $\ell$, we may
compare them as:
\begin{table}[!htbp]
	\renewcommand\arraystretch{1.2}
	\caption{Comparison for MgNet and ResNet type iterative forms }
	\label{comparison-ALL}
	\begin{center}\scriptsize
		\resizebox{.9\textwidth}{!}{
			\begin{tabular}{|c|c|c|}
				\hline
				Primal-Dual & Model & Iterative forms \\
				\hline
				\multirow{5}{*}{Feature space} & Abstract-MgNet & Solving $A^\ell(u^\ell) = f^\ell$ \\
				\cline{2-3}
				& Single step MgNet & $u^{\ell,i} = u^{\ell, i-1} + B^{\ell,i} (f^\ell - A^{\ell}(u^{\ell,i-1}))$ \\
				\cline{2-3}
				& Multi-step MgNet & $u^{\ell,i} = \sum_{j=0}^{i-1} \alpha^{\ell,i}_j ( u^{\ell,j} + B^{\ell,i}_j  ( {f^\ell -  A^{\ell} (u^{\ell,j})} ) )$ \\
				\cline{2-3}
				& Chebyshev-semi MgNet & $u^{\ell, i} = \omega^{\ell,i}(u^{\ell,i-1} + 
				B^{\ell,i}(f^\ell- A^\ell(u^{\ell,i-1})))+ (1- \omega^{\ell,i}) u^{\ell,i-2}$ \\
				\cline{2-3} 
				& {MgNet} & $u^{\ell,i} = u^{\ell, i-1} + \sigma \circ \eta^{\ell,i}\circ \sigma (f^\ell - \xi^{\ell}(u^{\ell,i-1}))$ \\
				\hline
				\multirow{5}{*}{Data space} & iResNet & $ f^{\ell,i} = f^{\ell, i-1} -  \xi^{\ell,i} \circ \sigma \circ \eta^{\ell,i} \circ \sigma (f^{\ell,i-1})$ \\
				\cline{2-3}
				& Mg-iResNet & $f^{\ell,i} = f^{\ell, i-1} -  \xi^{\ell} \circ \sigma \circ \eta^{\ell,i}\circ \sigma (f^{\ell,i-1})$ \\
				\cline{2-3}
				& Mg-ResNet & $f^{\ell,i} = \sigma(f^{\ell,i-1}) - \xi^{\ell} \circ \sigma \circ \eta^{\ell,i}\circ \sigma( f^{\ell,i-1})$ \\
				\cline{2-3}
				& $\sigma$-ResNet & $f^{\ell,i} = \sigma(f^{\ell,i-1}) - \xi^{\ell,i} \circ \sigma \circ \eta^{\ell,i}\circ \sigma( f^{\ell,i-1})$ \\
				\cline{2-3}
				& ResNet & $f^{\ell,i} = \sigma(f^{\ell, i-1} -  \xi^{\ell,i} \circ \sigma \circ \eta^{\ell,i} (f^{\ell,i-1}))$ \\
				\hline
			\end{tabular} 
		}
	\end{center}
\end{table}

We can have these connections for all iterative scheme in data space:
\begin{equation}
\text{ ResNet} \xleftrightarrow{\eqref{tilde-f}} \sigma\text{-ResNet } \xleftrightarrow{\xi^{\ell,i} \leftrightarrow \xi^{\ell}} \text{Mg-ResNet}  
\xleftrightarrow{\sigma(f^{\ell,i-1}) \leftrightarrow f^{\ell, i-1} } \text{Mg-iResNet} \xleftrightarrow{ \xi^{\ell} \leftrightarrow \xi^{\ell,i}} \text{iResNet}.
\end{equation}


In this sense, these MgNet related models can be understood as
models between iResNet and ResNet. And all these models can be
understood as iteration in the data space as a dual relationship with
feature space as MgNet.

The rationality of replacing  $\xi^{\ell,i}$ by layer independent $\xi^{\ell}$ may
be justified by the following theorem. 
\begin{theorem}\label{thm:CNN}
	On each grid $\mathcal T_\ell$, 
	\begin{enumerate}
		\item Any CNN model with
		\begin{equation}
		\label{CNN1}
		f^{\ell,i} =   \chi^{\ell,i} \circ \sigma (f^{\ell,i-1}),
		\end{equation} 
		can be written as
		\begin{equation}\label{Res-CNN1}
		f^{\ell,i} = \sigma(f^{\ell,i-1}) - \xi^{\ell} \circ \sigma \circ \eta^{\ell,i} \circ\sigma ( f^{\ell,i-1}).
		\end{equation}
		\item Any CNN model with 
		\begin{equation}
		\label{CNN2}
		f^{\ell,i} =   \sigma\circ\chi^{\ell,i} (f^{\ell,i-1}).
		\end{equation}
		can be written as 
		\begin{equation}\label{Res-CNN2}
		f^{\ell,i} = \sigma\left(f^{\ell,i-1} - \xi^{\ell} \circ \sigma \circ \eta^{\ell,i}  ( f^{\ell,i-1})\right).
		\end{equation}
	\end{enumerate}
	
\end{theorem}
\begin{proof}
	Let use prove the first case as an example, 
	the second case can be proven with the same process.
	
	With similar structure in MgNet, we can take
	\begin{equation}
	\label{xi-cnn1}
	\xi^{\ell}= \hat \delta^\ell := [\hat \delta_1, \cdots, \hat \delta_{{c_\ell}}],
	\end{equation}
	and 
	\begin{equation}
	\label{eta-ell}
	\eta^{\ell,i} = [{\rm id}_{c_\ell}, -{\rm id}_{c_\ell}] \circ (\chi^{\ell,i} - {\rm id}_{c_\ell}).
	\end{equation}
	Here 
	\begin{equation}
	{\rm id}_{c_\ell}: \mathbb{R}^{n_\ell \times n_\ell \times c_\ell} 
	\mapsto \mathbb{R}^{n_\ell \times n_\ell \times c_\ell},
	\end{equation}
	is the identity map and 
	\begin{equation}
	\hat \delta_k :  \mathbb{R}^{n_\ell \times n_\ell \times 2c_\ell} 
	\mapsto \mathbb{R}^{n_\ell \times n_\ell},
	\end{equation}
	with 
	\begin{equation}\label{eq:hatdelta}
	\hat \delta_k([X ,Y]) = -([X]_k + [Y]_k),
	\end{equation}
	for any $X, Y \in \mathbb{R}^{n_\ell \times n_\ell \times c_\ell}$ 
	and $[X,Y] \in \mathbb{R}^{n_\ell \times n_\ell \times 2c_\ell} $.

	First, we see that $\eta^{\ell,i}$ with the above 
	form is a convolution from $\mathbb{R}^{n_\ell \times n_\ell \times c_\ell}$
	to  $\mathbb{R}^{n_\ell \times n_\ell \times 2c_\ell}$.
	Following the identity
	\begin{equation}
	ReLU(x) + ReLU(-x) = x,
	\end{equation}
	and the definition of $\xi^{\ell}$ i.e. 
	\begin{equation}
	\xi^{\ell} = \hat \delta^\ell,
	\end{equation}
	as a special case in MgNet. 
	For more details, we can give a exact form of 
	$\hat \delta_k$ as in \eqref{eq:hatdelta} with
	\begin{equation}
	\hat \delta_k = [0, \cdots,0, -\delta, \cdots 0;  0, \cdots,0, -\delta, \cdots 0],  \quad k = 1:{c_\ell},
	\end{equation}
	where $\delta$ is the identity kernel in one channel.
	
	Furthermore, we have
	\begin{equation}
	\begin{aligned}
	\left[\xi^{\ell} \circ \sigma \circ [{\rm id}_{c_\ell}, -{\rm id}_{c_\ell}] (x) \right]_k &=  \left[\xi^{\ell} \circ \sigma \circ [x, -x]  \right]_k \\
	&= \hat \delta_k ( [\sigma(x), \sigma(-x)])  \\
	&= -\delta([\sigma(x)]_k) - \delta([\sigma(-x)]_k)\\
	&=-( \sigma([x]_k)+ \sigma(-[x]_k)) \\
	&=  -[x]_k.
	\end{aligned}
	\end{equation}
	Thus to say,
	\begin{equation}
	\xi^{\ell} \circ \sigma \circ [{\rm id}_{c_\ell}, -{\rm id}_{c_\ell}]  = -{\rm id}_{c_\ell}.
	\end{equation}
	Then the modified dual form of MgNet in \eqref{tilde-resnet} becomes
	\begin{equation}
	\begin{aligned}
	f^{\ell,i} &= \sigma(f^{\ell,i-1}) - \xi^{\ell,i} \circ \sigma \circ \eta^{\ell,i} \circ\sigma ( f^{\ell,i-1}) \\
	&=  \sigma(f^{\ell,i-1}) - \left( \xi^{\ell} \circ \sigma \circ [{\rm id}_{c_\ell}, -{\rm id}_{c_\ell}] \right) 
	\circ (\chi^{\ell,i} - {\rm id}_{c_\ell})\circ \sigma(f^{\ell,i-1})\\
	&=\sigma(f^{\ell,i-1}) + (\chi^{\ell,i} -{\rm id}_{c_\ell})\circ \sigma(f^{\ell,i-1})  \\
	&=\chi^{\ell,i} \circ  \sigma (f^{\ell,i-1}).
	\end{aligned}
	\end{equation}
	This covers \eqref{Res-CNN1}.
\end{proof}

\begin{remark}
	Theorems~\ref{thm:CNN} shows that general CNN in
	the forms of either \eqref{CNN1} or \eqref{CNN2} can be written recast
	as \eqref{Res-CNN1} or \eqref{Res-CNN2} with the data-feature mapping 
	$A^\ell=\xi^\ell$ that is not only independent of the layers, but is
	actually given a priori as in \eqref{xi-cnn1}.  In
	view of Theorems~\ref{thm:mgnet1} and \ref{thm:CNN}, the classic
	CNN models can be essentially recovered from MgNet by choosing
	$\xi^\ell$ a priori as in  \eqref{xi-cnn1}.  
	We believe that general and well-defined mathematical structure of MgNet would
	provide mathematical insights for understanding and developing these CNN models.
\end{remark}


\section{Numerical experiments}\label{sec:numerics}
In this section, we present some numerical results to illustrate the
efficiency and potential of MgNet as described in Algorithm
\ref{alg:mgnet}.

\subsection{Data sets and model structure }
We choose CIFAR-10 and CIFAR-100 
\cite{krizhevsky2009learning}
as two data sets for numerical tests. 
Here, the CIFAR-10 dataset consists of 60000 32x32 color 
images in 10 classes, with 6000 images per class. 
The CIFAR-100 dataset is just like the CIFAR-10, 
except it has 100 classes containing 600 images each. 
We split these two data sets with 50000 training images 
and 10000 test images.

We will mainly carry out
a comparison with study between MgNet and  ResNet \cite{he2016deep} 
on these two data sets, so we choose some 
similar process techniques in ResNet such as there will 
be a average pooling before linear regression
layers:
\begin{equation}\label{eq:ave-pooling}
R_{ave}: \mathbb{R}^{m_{J-1} \times n_{J-1} \times c_{J-1}} \mapsto \mathbb{R}^{c_{J-1}}.
\end{equation}
Here, we can recover this average operator by taking $\nu_{J} = 0$ in MgNet and
$$
u^{J} = u^{J,0} = \Pi_{J-1}^J u^{J-1, \nu_{J-1}} \in \mathbb{R}^{c_{J-1}},
$$
with
$$
\Pi_{J-1}^J  = R_{ave}.
$$
This can be true also thanks to our structure that 
\begin{equation}\label{eq:c_u}
c_{u,\ell} = c_{u}, \quad 1 \le \ell \le J.
\end{equation}
Given an image $f$, similar to ResNet, we apply our MgNet as follows:
\begin{equation}\label{final-mg}
y = S \circ \theta \circ u^{J}(f),
\end{equation}
where $u^J(f)$ is the output from our MgNet as described in Algorithm
\ref{alg:mgnet},  $S$ is the soft-max mapping in \eqref{softmax} and 
\begin{equation}\label{final-theta}
\theta: \mathbb{R}^{c_u} \mapsto \mathbb{R}^\kappa,
\end{equation}
represents a fully linear layer with $\kappa = 10$ for CIFAR-10 and 
$\kappa = 100$ for CIFAR-100.

We will make the following choice of hyperparameters
for the MgNet:
\begin{itemize}
	\item $f_{\rm in}$: data initialization process. Similar to ResNet, we take 
	$f_{\rm in}(f) = \sigma \circ \theta^0(f)$ as discussed in \S~\ref{sec:ini-mgnet} and \S~\ref{sec:CNNs}.
	\item $J$: the number of grids. As all images in CIFAR-10 or CIFAR-100
	are $32\times 32 \times 3$,  we choose $J = 5$ to be consistent with ResNet.
	\item $\nu_\ell$:  the number of smoothings in each grids. To be consistent with
	ResNet-18 or ResNet-34 we choose $\nu_\ell = 2$ or $\nu_\ell = 4$.
	\item $c_u$ and $c_f$: the number of feature and data channels. 
	\item $A^\ell$: the data-feature mapping. We choose the linear case in \eqref{linearA}.
	\item $B^{\ell,i}$: the feature extractor. We choose the variable extractors as in \eqref{extractor-ell}.
	\item $R_{\ell}^{\ell+1}$: the restriction operator in \eqref{restrict-f}. 
	Here we choose it as a convolution with stride $2$ which need to be trained.
	\item $\Pi_\ell^{\ell+1}$: the interpolation operator in
	\eqref{interpolation}.  Here we compare these next three
	different choices: 
	\begin{enumerate}
		\item {$\Pi_0$: } $\Pi_\ell^{\ell+1} = 0$;
		\item {$\Pi_1$: }convolution with stride $2$ which need to be
		trained; 
		\item {$\Pi_2$: }channel-wise interpolation as in
		\eqref{Pi}, with $\bar \Pi_{\ell}^{\ell+1}$ as a convolution
		with one channel and stride $2$ which also need to be trained. 
	\end{enumerate}
\end{itemize}

\subsection{Training algorithm}
While there are many different choices of training algorithms \cite{bottou2018optimization}, 
in our test, we adopt the popular 
stochastic gradient descent (SGD) with mini-batch and momentum for
cross-entropy loss function.
\begin{breakablealgorithm}
	\caption{SGD with mini-batch and momentum}
	\label{alg:sgd}
	\begin{algorithmic}
		\State {\bf Input}: learning rate $\eta_t$, batch size $m$, parameter Initialization $ w_0$, number of epochs $K$. 
		\For{Epoch $k = 1:K$} \\
		\State Shuffle data and get mini-batch $B_1, \cdots, B_{\frac{N}{m}}$, choose mini-batch as: $B_{i_t}$ with
		$$
		i_t \equiv t \mod(\frac{N}{m}),
		$$
		\State Compute the gradient on $B_{i_t}$:
		$$
		g_t = \nabla_{w} \frac{1}{m} \sum_{i \in B_{i_t}} h_i(w_{t}).
		$$
		\State Compute the momentum:
		\begin{equation}
		v_t = \alpha v_{t-1} - \eta_t g_t \quad (v_0 = 0).
		\end{equation}
		\State Update $w$:
		\begin{equation}
		w_{t+1} = w_t + v_t.
		\end{equation}
		\EndFor
	\end{algorithmic}
\end{breakablealgorithm}

Here we have $h_i(w_t) = l(\classmap(f_i;w_t),y_i)$ as defined in \eqref{eq:3}, where $w_t$ notes all free parameters in MgNet and $\theta$ in \eqref{final-theta}.
We use the SGD with momentum of 0.9. 
The mini-batch size is chosen as
$m=128$. The learning rate starts from 0.1 and is divided by $10$ for
every $30$ epochs, and the models are trained for up to $K=120$ epochs.
We adopt batch normalization (BN) after each convolution and before
activation, following \cite{ioffe2015batch}.  Initialization strategy
is the same with ResNet as in \cite{he2015delving}.  We
do not use weight decay and dropout.  The final Top-1 test accuracy is
shown in Table~\ref{comparison}.
\begin{table}[!htbp]
	\caption{ResNet and MgNet on CIFAR-10 and CIFAR-100. 
		Our methods are named with $\nu_\ell$, ($c_u$, $c_f$), $\Pi_\ell^{\ell+1}$ by definition above.}
	\label{comparison}
	\vskip 0.15in
	\begin{center}
		\begin{tabular}{cccc}
			\hline
			Models & CIFAR-10 & CIFAR-100 & Params \\
			\hline
			ResNet-18 & 92.24 & 71.96 & 11.2M   \\
			ResNet-34 & 92.80 & 71.93 & 21.3M   \\
			\hline
			$2, (256,256)$, $\Pi_0$ & 92.02 & 68.29 & 7.1M  \\
			$2, (256,256)$, $\Pi_1$ & 93.04 & 72.32 & 8.9M  \\
			$2, (256,512)$, $\Pi_1$ & 93.20 & 72.42 & 19.5M  \\ 
			$2, (256,512)$, $\Pi_2$ & 93.53& 74.26 & 17.7M  \\ 
			\hline
		\end{tabular} 
	\end{center}
	\vskip -0.1in
\end{table}

From the above numerical results, we find that the modified CNN models
based on MgNet structure have competitive and sometimes better
performance in comparison with standard ResNet models when applied to
both CIFAR-10 and CIFAR-100 data sets. Generally speaking, the more
channels the better performance you can achieve (see WideResNet
\cite{zagoruyko2016wide} for similar observation). Furthermore,
$\Pi_1$ and $\Pi_2$ work better than $\Pi_0$, and $\Pi_2$ can even
work better than $\Pi_1$ with fewer parameters for big enough 
channel numbers.

\section{Concluding remarks}\label{sec:conclusion}
By carefully studying the connections between the traditional
multigrid method and the convolutional neural network (especially the
ResNet type) models, the MgNet established in this paper provides a
unified framework that connects both multigrid and CNN in a technical
level.  Comparing with other existing works that discuss the
connection between multigrid and CNN, MgNet goes beyond formal or
qualitative comparisons and identifies key model components that play
the same corresponding roles, from an abstract viewpoint, for these two different
methodologies.  As a result, how and why CNN models work can be
mathematically understood in a similar fashion as for multigrid method
which has a much more mature and better developed theory.  Motivated
from various known techniques from multigrid method, many variants and
improvements of CNN can then be naturally obtained.  For example, as
demonstrated from our preliminary numerical experiments, the resulting
modified CNN models equipped with fewer weights and hyperparameters
actually exhibit competitive and sometimes better performance than
standard ResNet models.

The MgNet framework opens a new door to the
mathematical understanding, analysis and improvements of deep learning
models.  The very preliminary results presented in
this paper have demonstrated the great potential of MgNet from both
theoretical and practical viewpoints.  Obviously many aspects of MgNet
should be further explored and expect to be much improved.  In fact, only very
few techniques from multigrid method have been tried in this paper and
many more in-depth techniques from multigrid require further study for
deep neural networks, especially CNN.  
In particular, we believe that the MgNet framework will
lead to improved CNN that only has a small fraction of the number
of weights that are required by the current CNN. On the other hand,
the techniques in CNN can also be used to develop new generation of multigrid
and especially  algebraic multigrid methods \cite{xu2017algebraic} for solving
partial differential equations. Our ongoing works have
demonstrated great potentials for research in these directions and  many
more results will be reported in future papers. 

\section*{Acknowledgement}
We would like to thank Xiaodong Jia for his help with the numerical experiments.
The work of the first author was supported in part by 
The Elite Program of Computational and Applied
Mathematics for PhD Candidates of Peking University.
The work of the second author was supported in part by 
the US National Science Foundation under Award Number DMS-1819157
and also by the US Department of 
Energy Office of Science, Office of Advanced Scientific Computing Research, 
Applied Mathematics program under Award Number DE-SC0014400.

\bibliographystyle{plainnat}      
\bibliography{MgNet-arXiv-1May}

\end{document}